\documentclass[accepted]{uai2026} 
                        
\usepackage[american]{babel}

\usepackage{natbib} 
\usepackage{mathtools} 
\usepackage{booktabs} 
\usepackage{tikz} 

\usepackage{graphicx}
\usepackage{amsmath}
\usepackage{amsthm}
\usepackage{booktabs}
\usepackage{algorithm}
\usepackage{algpseudocode}
\usepackage{calrsfs}
\usepackage{pifont} 
\usepackage{xcolor}
\usepackage{amsfonts}

\newtheorem{theorem}{Theorem}
\newtheorem{definition}{Definition}
\newtheorem{example}{Example}

\title{Lifted Relational Probabilistic Inference via Implicit Learning}

\author[1]{Luise Ge}{}
\author[1]{Brendan Juba}{}
\author[1]{Kris Nilsson}{}
\author[1]{Alison Shao}{}

\affil[1]{%
    Computer Science $\&$ Engineering\\
    Washington University in St. Louis\\
    St. Louis, Missouri, USA
}

\begin{document}
\maketitle

\begin{abstract}
  Reconciling the tension between inductive learning and deductive reasoning in first-order relational domains is a longstanding challenge in AI. We study the problem of answering queries in a first-order relational probabilistic logic through a joint effort of learning and reasoning, \textit{without ever constructing an explicit model}.  Traditional lifted inference assumes access to a complete model and exploits symmetry to evaluate probabilistic queries; however, learning such models from partial, noisy observations is intractable in general. We reconcile these two challenges through 
\textit{implicit learning to reason} and \textit{first-order relational probabilistic inference} techniques. More specifically, we merge incomplete first-order axioms with independently sampled, partially observed examples into a bounded-degree fragment of the sum-of-squares (SOS) hierarchy in polynomial time. Our algorithm performs two lifts simultaneously: (i) \emph{grounding-lift}, where renaming-equivalent ground moments share one variable, collapsing the domain of individuals; and (ii) \emph{world-lift}, where all pseudo-models (partial world assignments) are enforced in parallel, producing a global bound that holds across all worlds consistent with the learned constraints. These innovations yield the first polynomial-time framework that implicitly learns a first-order probabilistic logic and performs lifted inference over both individuals and worlds.
\end{abstract}

\section{Introduction}

Many real-world domains—from scientific databases to social networks—involve reasoning about uncertain relationships among entities. Probabilistic relational reasoning addresses the challenge of drawing sound inferences in such settings, where both the relational structure (who relates to whom, which properties hold) and the probabilistic dependencies (how likely various configurations are) must be represented and reasoned about jointly. This requires combining the expressiveness of first-order logic, which naturally captures relational structure through quantified variables and predicates, with probabilistic semantics that quantify uncertainty. However, this combination faces a nontrivial challenge: expressive logical languages lead to computational intractability in large domains.

Probabilistic relational models have been  developed to address probabilistic relational inferences. Markov Logic Networks (MLNs) \citep{richardson2006markov}, Probabilistic Relational Models (PRMs) \citep{getoor2001learning}, and probabilistic logic programming frameworks  \citep{de2007problog,bruynooghe2010problog,manhaeve2018deepproblog} define distributions over relational structures using weighted logical rules or probabilistic programs. While lifted inference techniques~\citep{van2011lifted} exploit symmetries to improve scalability, these approaches share a fundamental limitation: they treat learning and inference as separate stages. Learning constructs a complete probabilistic model from data, then inference operates on that fixed model. This modularity is brittle—learning expressive models is often NP-hard~\citep{koller2009probabilistic}, and partial observations obscure true dependencies. Even approximate models frequently fail to preserve the symmetries and sparsity that enable tractable lifted reasoning, resulting in systems that are both computationally expensive and structurally misaligned with inference requirements.

Recently, a distinct approach to probabilistic inference has emerged through sum-of-squares (SOS) logic~\citep{juba2019polynomial,ge2025polynomial}. Rather than constructing an explicit distribution and performing approximate inference, SOS checks whether the knowledge base constraints can be simultaneously satisfied by any probability distribution. This is formulated as a semidefinite program that searches for polynomial refutations—if a bounded-degree refutation exists, the constraints are provably inconsistent. When such a refutation exists using polynomials of degree at most $d$, it can be found in polynomial time (for fixed $d$), providing formal certificates rather than approximate solutions.

However, these methods still assume access to a complete, explicitly specified model, leaving the challenge of learning from partial and noisy observations unaddressed. We introduce a new framework for implicit probabilistic relational reasoning that sidesteps the need to construct an explicit probabilistic model. Instead, our method integrates background knowledge and partial observations on the fly, while still enabling sound and tractable probabilistic reasoning over relational structures.

\textbf{Our Contribution} Our result may be seen as drawing on and combining the key features of two previously incomparable approaches to implicit learning. The first performed implicit learning in probabilistic inference, using the propositional sum-of-squares system \citep{juba2019polynomial}. The second achieved implicit learning in first-order Boolean logics \citep{belle2019}. 

The combination is nontrivial due to the following technical issue: the approach used by Juba for probabilistic inference embedded the empirical estimates into the semidefinite program relaxation used to perform the inference. By contrast, Belle and Juba leveraged Belle's grounding trick for first-order inference \citep{belle2017open}: they enumerated grounding sets for each of the partial observations, searching for an appropriate grounding for each observation. It is not immediately clear how to write constraints for the empirical expressions of the semidefinite program used by Juba that capture the appropriate choice of groundings for each example; na\"{\i}vely, it might seem to correspond to an exponential number of constraints, one for each combination of grounding sets for each of the partial examples. 

This impasse is broken by observing that fixing the grounding set across all examples is necessary to obtain valid empirical estimates of the individual marginal probabilities. For example, in a lottery, we want bounds on the probability of each ticket winning, whereas if we could change the grounding considered in each draw, we could end up with an overly optimistic estimate that uses the winning ticket each time. So we solve a semidefinite program for each of the grounding sets, in which the same grounding is used across all of the empirical expressions. As a consequence, we also end up with formally distinct notions of ``witnessing'' (the criteria for inclusion in the implicit knowledge base) of the two types of knowledge in our probabilistic reasoning, expectation bounds and support constraints. Indeed, the former are learned by empirical estimates with confidence intervals (which must use the consistent set of ground variables), whereas the latter can be refuted by the existence of a single violating grounding. This complication arises due to the interaction of the two key features here, of probabilistic expressions and the relational language.

\subsection{Related Work}

\paragraph{Inference Approaches}
Tractable Markov Logic (TML) and Weighted First-Order Model Counting (WFOMC) are both important frameworks for reasoning under uncertainty. They appear similar to our model, but they operate in settings where an explicitly specified model is assumed. A more recent similar approach is that of {Ge et. al.} All of these approaches are focused on tractable inference within such models, not learning, hence they are not direct comparisons to our model.

\paragraph{Model‐Free Approaches}  
Model‐free methods answer queries without constructing a full probabilistic model. The closest antecedents are the implicit learning‐to‐reason frameworks of \citet{juba2013implicit,juba2019polynomial}, \citet{belle2019}, \citet{mocanu2020polynomial}, and \citet{rader2021learning}. However, none of these extend to first‐order probabilistic relational domains. Our approach generalizes implicit reasoning to first-order Probability Relational Logic (FOPRL)~\cite{ge2025polynomial} via a lifted Sum‐of‐Squares relaxation, enabling tractable inference with \emph{double lifting} (grounding‐lift and world‐lift) directly on axioms and data.

\paragraph{Classical Model‐Based Approaches}  
Probabilistic relational models \citep{richardson2006markov,getoor2001learning} define joint distributions over relational structures via weighted logical rules. Although expressive, they suffer from intractable inference and expensive structure learning under partial observability. Lifted inference techniques~\citep{van2011lifted} exploit symmetries to reduce grounding, but still assume a complete model and face worst‐case hardness. ILP methods~\citep{cropper2022inductive} learn symbolic rules from data but typically ignore uncertainty or require extensive hand‐crafting.

\paragraph{Neural‐Symbolic Approaches}  
DeepProbLog~\citep{manhaeve2018deepproblog} extends ProbLog with trainable neural predicates, but relies on grounding and Sentential Decision Diagram (SDD) compilation, limiting scalability in evolving domains. Neural Markov Logic Networks (NMLNs)~\citep{marra2021neural} replace symbolic rules with neural potentials over fragments, learning a smooth approximation of the joint distribution. Both remain fundamentally model‐based—requiring sampling or approximate inference over an explicit generative model—and so do not eliminate the bottlenecks of structure learning or grounding. Approaches like Neural Theorem Provers~\citep{rocktaschel2017end}, Logic Tensor Networks~\citep{donadello2017logic} similarly attempt to reconcile symbolic reasoning with neural networks, but ultimately face similar model-based limitations requiring approximate inference or sampling over an explicit representation.

\paragraph{Large Language Models, Reasoning Models, and so on}
An influential paper on language models (specifically focusing on the then-new GPT-4) suggested that these models captured reasoning ability \cite{bubeck2023sparks}. In this approach, a large neural network is simply trained to predict the next token in a corpus of natural language text; by setting up a reasoning problem and asking the model to predict the answer, these models can be coerced to solve reasoning tasks. Numerous groups observed that these models perform unreliably on such tasks \citep{dziri2023faith,mccoy2023embers,gendronlarge}. Subsequently, new ``reasoning'' models such as Open AI's o1 and DeepSeek's R1 \citep{guo2025deepseek} were released, based on requiring the language model to output a ``chain of thought'' towards responding to a query. These new systems do not fundamentally address the basic limitations of the previous generation of language models \citep{mccoy2024language,kambhampati2025reasoning}. The systems are extremely large, costly to train and use, opaque, and unreliable. We aim to do better in all of these respects. Furthermore, in comparison to the approach we consider here, these systems do not produce estimates of probabilities or expected values for a query---at least, not in any straightforward way.

\section{Preliminaries}\label{sec:preliminary}

We study probabilistic inference over relational structures by assuming a distribution $D$ over all possible relations among entities. Although this distribution is unknown, we assume a knowledge base that provides partial characterization about the distribution through two types of constraints: logical constraints that specify which restrictions the relations must hold almost surely, and moment constraints restricting expected values of the combinations of relations that must be satisfied. To express these constraints, we represent both the relations and their moments as variables and formulate the constraints as polynomial inequalities. We begin by defining the language used to formalize these expressions.

\textbf{Language:} We work in a first-order relational language $\mathcal{L}$ that includes a countably infinite set of \emph{names} $\mathcal{N} \cong \mathbb{N}$ serving as the domain of quantification. Formally, the set $\mathcal{N}$ will correspond to the set of integers $\mathbb{N}$, but informally we will use proper names such as $\{adam, alex, aaron,...\}$ for readability. The language also includes a countable set of \emph{variables} $\mathcal{V}=\{u,v,w,\dots\}$ that can represent arbitrary elements in the domain, and \emph{relational symbols} of various finite arities (e.g., $\{P(v), Q(u,v),...\}$). Importantly, we allow the relations to take \textbf{real values}. We also consider a finite set of \emph{constants} $\mathcal{C} \subseteq \mathcal{N}$, which are fixed elements of the domain. We refer to $\mathcal{G} = \mathcal{N} \setminus \mathcal{C}$ as \emph{generic names} that can be freely permuted without affecting the constants. A \emph{renaming substitution} is any permutation on $\mathcal{G}$ that fixes $\mathcal{C}$. 
A \emph{term} is $R(t_1,\dots,t_k)$ with $R$ a $k$-ary predicate and each $t_i$ is either a variable or a name.  If all $t_i\in\mathcal N$, we call the term an \emph{atom}. We write
$
  \mathrm{ATOMS}
  = 
  \{\,R(a_1,\dots,a_k)\mid R\text{ is $k$-ary},\;a_i\in\mathcal N\}.
$

\textbf{Logical Constraints:} We can think of each term $\tau\in\mathcal T$ as a random variable $x_\tau$ in our framework that represents the value of a relation
. In particular, the distribution we aim to reason about is a joint distribution over all ground (i.e., fully instantiated) relations. To express more complicated interactions among relations, we allow monomials—products of these variables raised to nonnegative integer powers, limited in degree to keep expressions manageable. Specifically,  fixing a maximum degree $d$, a \emph{monomial} of degree at most $d$ is $ \mu
  :=
  \prod_{\tau \in \mathcal{T}} x_\tau^{\alpha_\tau},
  \quad
  \alpha_\tau\in\mathbb Z_{\ge0},
  \quad
  \sum_i \alpha_\tau \le d.$
Given a finite set $M$ of such monomials and real coefficients $\{c_\mu\}$, a polynomial inequality over terms or alternatively a linear inequality over their monomials is
$
 p(\vec{x})= \sum_{\mu\in M} c_\mu\,\mu \ge 0,
  \quad
  \max_{\mu\in M}\deg(\mu)\le d.
$
We bind these inequalities under universal quantifiers analogous to the language of \citet{lakemeyer2002evaluation}, whose domain is carved out by Boolean combinations of equalities (e.g.,\ $u=v$, $x=\mathit{adam}$) over $\mathcal V\cup\mathcal C$.

By pairing an equality expression $\Xi$ with a polynomial inequality $\Lambda$ we can obtain a \emph{logical constraint} formula $\Phi = \forall \Xi \supset \Lambda$. This can be shortened to $\forall \Lambda$ when $\Xi$ is a tautology. The \emph{quantifier rank} of $\Phi$ can be given by the number of distinct variables occurring in $\Xi$ and $\Lambda$. The semantics of $\Phi$ is that given a substitution of names $\theta$ satisfies $\Xi$ then that same substitution will satisfy $\Lambda$.

\begin{example}
    The logical constraint $\forall x. height(x)\ge 0$ asserts that every individual must have nonnegative height.
\end{example}

\textbf{Expectation Constraints:} To express probabilistic knowledge, we introduce expectation constraints in the spirit of \citet{halpern2007characterizing}. For any monomial $\mu$, let $e(\mu)$ denote its expected value under the unknown distribution $D$. A (first-order) expectation constraint is a linear inequality over moment variables in the form of: $\forall\,\Xi\;\supset\;\sum c_{e(\mu)}\,e(\mu)\,\{\ge,\le\}c_0.$ Its quantifier rank and semantics mirror those of logical constraints.

\begin{example}
    The expectation constraint $\forall x \neq alex \supset e(\textup{height}(x)) -60 \geq0$
     asserts that the average height of all individuals, except Alex, is over 60 inches.
\end{example}

\textbf{Knowledge Base and Query:}
A \emph{knowledge base} $\Delta$ is a finite, nonempty set of logical and expectation constraints.  A \emph{query} $\varphi$ is another constraint whose validity we test by checking consistency of $\Delta\cup\{\varphi\}$ (we abbreviate this union simply as~$\Delta$).

\textbf{Semantics:}
A $\mathcal{L}$-model $\vec{x}$ assigns real-valued functions to each relational symbol over the domain $\mathcal{N}$, along with a probability measure on this space, such that all logical constraints hold almost surely and all expectation constraints hold exactly.

\textbf{Grounding:}
A ground knowledge base $GND(\Delta)$ is one that only contains \textit{ground} terms. Starting with a general $\Delta$ we can obtain a $GND(\Delta)$ by substituting variables with names. We will also define the rank of $\Delta$ to be the maximum quantifier rank of any formula in $\Delta$. This allows us to define $GND(\Delta) = \{\phi \theta | [\forall \Xi \supset \phi] \in \Delta, \models \Xi\theta\}$. In an open universe where we have a countably infinite number of names the grounding can also be infinite. This makes the following form more useful: $GND(\Delta,k) = \{\phi \theta | [\forall \Xi \supset \phi] \in \Delta, \models \Xi\theta\, \theta \in \mathcal{C}\}$ where $\mathcal{C}$ consists of all constants present in $\Delta$ as well as $k$ additional generic names. This definition limits our substitutions to a finite subset of our possible names. Specifically, one that contains all constants, and $k$ additional generic names.

\begin{example}
    Given $\Delta = \{\forall x P(x,alex) \geq 0\}$ then $GND(\Delta,2) = \{P(alex,alex)\geq0,P(adam,alex)\geq0, P(aaron,alex)\geq0\}$.
\end{example}

\subsection{Lifted Relational Sum-of-Squares Refutations}
We now recall the construction of lifted relational SOS refutations introduced by~\citet{ge2025polynomial}. These refutations enable reasoning over knowledge bases consisting of both logical and expectation constraints by demonstrating that no joint distribution—that is, no model—can satisfy all the constraints simultaneously. This approach extends Putinar's Positivstellensatz~\citep{putinar} to the relational setting.  

Let $\Delta$ consist of a collection of logical inequalities $\{g_i \geq 0\}_{i \in I}$, and expectation constraints $\{b_j\geq 0\}_{j\in J}$. A \emph{sum-of-squares polynomial} is given by a sum of squares of polynomials $\sum_{\ell \in L}(p_{\ell})^2$.  A \emph{sum-of-squares refutation} is given by the sum-of-squares polynomials $\sigma_0$ and $\sigma_i$ for ${i \in I}$, and positive constants $\{r_j\}_J$ satisfying the following formal equation:

\begin{equation}
\small
    \sigma_0+\sum_{i\in I}\sigma_ig_i+\sum_{j\in J}r_jb_j
=-1 \label{e:sos_proof} 
\end{equation}

Intuitively, this is a sound refutation: if each component of the sum is non-negative under any distribution consistent with the knowledge base, their sum cannot equal $-1$. Thus, Equation~\ref{e:sos_proof} certifies that no such distribution exists---$\Delta$ is inconsistent. 
 
This converts probabilistic inference into an algebraic feasibility problem: finding $\sigma_0$ and $\sigma_i$ for ${i \in I}$, along with positive constants $r_j$ that satisfy Equation~\ref{e:sos_proof}. 
Assuming the system is ``explicitly compact'' (see Definition 2), refutations may be found in polynomial time so long as the degree of the refutation is bounded by an absolute constant \citep{shor87,nesterov00,parrilo00,lasserre01}.

This is achieved by solving a semidefinite program constructed as follows: first, supposing $\vec{v}$ is a vector of monomials of degree up to $d/2$, abusing notation, let $e(\vec{v}\vec{v}^\top)$ denote the matrix obtained by taking the $e(\mu)$ variable for each monomial $\mu$ in the matrix of variables $\vec{v}\vec{v}^\top$, known as the \emph{moment matrix}. For a polynomial $h_j$ of degree $d_j$, for each monomial $\mu$ of degree up to $d-d_j$, we consider the \emph{shifted} polynomial $\mu h_j$ 
in which each monomial $\nu$ of $h_j$ is substituted by the variable $e(\mu\cdot\nu)$ (obtained by summing the degree vectors).
Next, for a given polynomial $g_i$ of degree $d_i$, let $v'$ be a vector of monomials of degree up to $(d-d_i)/2$, and let $e(g_i\vec{v'}\vec{v'}^\top)$ denote the matrix of shifts of $g_i$ with the monomial $\nu$ is the monomial that would appear in the corresponding entry of the moment matrix  $e(\vec{v'}\vec{v'}^\top)$. We refer to this as the \emph{localizing matrix} for $g_i$. Now, the \emph{sum-of-squares program} is the semidefinite program that asserts that the moment matrix is positive semidefinite, the localizing matrices for the inequality constraints are positive semidefinite, the shifted polynomials for the equality constraints are all equal to 0, and the expectation bounds are simply included with each monomial $\mu$ replaced by the corresponding $e(\mu)$.

The lifted variant is obtained by grounding $\Delta$, and showing that the ground refutation is sound and complete for the lifted model.
\begin{definition}[\cite{ge2025polynomial}] A \emph{lifted sum-of-squares system} is given by the union of $GND(\Delta,k)$ and the set of equality constraints $e(\mu) - e(\mu') = 0$ for each pair of ground monomials used in $GND(\Delta,k)$ such that for a renaming substitution $\theta$, $\mu = \mu'\theta$. 

\end{definition}

\begin{theorem}[\citet{ge2025polynomial}]  
\textup{(Transferability from open to closed universe)}  
For any fixed degree $d$, the grounded system $GND(\Delta)$ admits a degree-$d$ sum-of-squares refutation if and only if the lifted sum-of-squares system over $GND(\Delta, k)$ admits a degree-$d$ refutation, where $k$ is the quantifier rank of $\Delta$.
\end{theorem}

This theorem guarantees that every $GND(\Delta)$ that uses an infinite set of names also has a refutation when limited to $GND(\Delta,k)$. For inference, now, we need the analogue of explicit compactness for our first-order systems:

\begin{definition}
    The system $\Delta$ is \emph{explicitly compact} if for each monomial $\mu$, $\Delta$ includes constraints of the form $\mu - L_{\mu} \geq 0$ and $U_{\mu} - \mu \geq 0$ for some $L_\mu, U_{\mu}$ that are finite.
\end{definition} 
Explicit compactness guarantees every $\mu$ falls within the bounded range $[L_{\mu},U_{\mu}]$.

\begin{theorem}\textup{(Soundness and Completeness, \cite{ge2025polynomial})}\label{thm:lifted-guarantee}
    Given an explicitly compact knowledge base $\Delta$ and a grounding $GND(\Delta, k)$, a lifted sum-of-squares program using $GND(\Delta,k)$ is sound and complete for degree-$d$ refutations of $\Delta$.
\end{theorem}

The second theorem ensures that lifted sum-of-squares is sound and complete on a $GND(\Delta,k)$. These two theorems together are sufficient to show that lifted sum-of-squares is sound and complete for an infinite set of names.

We remark that under certain conditions, some Sum-of-Squares (SOS) program with finite degree $d$ is equivalent to the infinite version (without a bound on the degree) \citep{ghasemi2016moment}. In this case, the program is only feasible if a consistent probability distribution exists, and the system is complete in a strong sense. The degree may be large, however, so this is generally not computationally useful.

\section{Formulation of Learning with Partial Observations}

In many real-world domains, especially those involving complex relational structures (e.g., healthcare, scientific experiments, or social networks), observations are often partial or incomplete. Crucially, these missing values are not merely noise — they arise from systematic masking processes (e.g., sensor limitations, privacy constraints, or study design). Our formulation explicitly models this partial observability by introducing a probabilistic masking process $\Theta$ over full relational models. This allows us to reason about uncertainty both from the underlying probabilistic structure and from the incompleteness of the data, in a unified framework. It also enables us to answer queries that are robust to missing information, rather than relying on imputation or assuming fully observed inputs.

\begin{definition}
A \textbf{partial model} $\vec{\rho}$ maps $\mathrm{ATOMS}$ to $\mathbb{R}\cup\{*\}$. We say $\vec{\rho}$ is consistent with a full model~$\vec{x} \in \mathcal{M}$ if $\vec{\rho}_\tau = \vec{x}_\tau$ for all coordinates such that $\vec{\rho}_\tau \ne *$. We call the set of partial models $\mathcal{P}.$
\end{definition}

\begin{definition}
A \textbf{masking function} is a function $\theta: \mathcal{M} \to \mathcal{P}$ such that $\theta(\vec{x})$ is consistent with $\vec{x}$. A \emph{masking process}~$\Theta$ is a random variable over such functions. The induced distribution over partial models is denoted~$\Theta(D)$.
\end{definition}

Our goal is to answer relational‐probabilistic queries using two sources of information: (i) a partial knowledge base~$\Delta$ specified in the first-order relational probabilistic logic from Section~\ref{sec:preliminary} describing $D$, and (ii) a collection of \emph{partial models} drawn i.i.d. from $\Theta(D)$. 


\begin{example}
Consider a laboratory studying several candidate drugs that might shrink tumors.  Their initial knowledge base $\Delta$ encodes logical constraints and a criterion for progressing to a second trial stage. 
\begin{equation*}
\begin{aligned}
&\forall x,d:\; \texttt{treated}(x,d)^2 = \texttt{treated}(x,d),\\
&\forall x:\; \texttt{shrunk}(x)^2 = \texttt{shrunk}(x),\\
&\forall d:\; \texttt{second\_stage}(d)^2 = \texttt{second\_stage}(d),\\
&\forall x,d:\;\bigl(\mathbb{E}[\texttt{treated}(x,d)\,\texttt{shrunk}(x)]
      \;-\;\\
& \qquad     0.8\,\mathbb{E}[\texttt{treated}(x,d)] \le 0\bigr)
      \;\implies\;\\
&\qquad \qquad \texttt{second\_stage}(d)=0
\end{aligned}
\end{equation*}

Suppose our samples come from a mixture of three ``worlds". Because of limitations in the imaging equipment, each observation always shows which drug was given but independently hides the value of
$\texttt{shrink}(x)$ with 20\% probability. The hypothesis under test is ``every drug advances to Phase II", or equivalently---  $\forall d, \texttt{second\_stage}(d)=1.$

\begin{table}[h]
\centering
\begin{tabular}{c c c c c c}
\toprule
\textbf{World} & \textbf{Prob.} & \textbf{$m_1 $D} & $\texttt{s}(m_1)$ & \textbf{$m_2$D} & $\texttt{s}(m_2)$ \\
\midrule
1 & 0.3 & A & 1 & B & 0 \\
2 & 0.5 & A & 0 & B & 0 \\
3 & 0.2 & A & 0 & B & 1 \\
\bottomrule
\end{tabular}
\caption{Trial outcomes in each world (shrink labels hidden w.p.\ 20\%).}
\end{table}

\end{example}

Under partial observability, we can only have partial evaluations over each polynomial constraint. 

\begin{definition}[Partial Evaluation]
    For a polynomial $p(\vec{x})$, we will let $p|_{\vec{\rho}}(\vec{x})$ denote the polynomial with $\rho_\tau$ plugged in for $x_\tau$ whenever
$\rho_\tau \ne *$, and collecting terms with the same monomial. We
refer to this as the partial evaluation of $p$ under $\vec{\rho}$.
\end{definition}

We can only utilize frequently \emph{witnessed} constraints—those that remain verifiable despite missing entries. Since monomial bounds are not directly observed, they must often be computed via sum-of-squares reasoning. On each example, witnessing is defined as:

\begin{definition}[Witnessing]
Let $\vec\rho$ be a partial model, and suppose that for each monomial $\mu$ the program constraints include some explicit upper and lower bound $U_{\mu}$ and $L_{\mu}$.  Let
$
L_{\mu}(\vec\rho)\;\le\;\mu\;\le\;U_{\mu}(\vec\rho)$
be the tightest lower and upper bounds on the monomial $\mu$ consistent with~$\vec\rho$ and the constraints of the given program.  We say a polynomial inequality
$p(\vec x)=\sum_{\mu\in M} c_\mu \mu \ge 0$ with $c_\mu\in \mathbb{R}$
is \emph{witnessed} by $\vec\rho$ if, after substituting each $\mu$ by its worst‐case value
\[
\hat \mu
\;=\;
\begin{cases}
L_{\mu}(\vec\rho), & c_{\mu}>0,\\
U_{\mu}(\vec\rho), & c_{\mu}<0,
\end{cases}
\]
the resulting constant remains nonnegative:
$
\sum_{\mu} c_{\mu}\,\mu
\;\ge\;0.
$
\end{definition}

\begin{definition}[Testability of Polynomial Constraints]
A system of polynomial inequalities $\{p_\ell(\vec{x}) \ge 0\}_{\ell \in L}$  is \emph{$(1-\gamma)$-testable} under $D$ and $\Theta$ if, with probability at least $1 - \gamma$ over $\vec{\rho} \sim \Theta(D)$, every $p_\ell$ is witnessed under~$\vec{\rho}$.
\end{definition}

\begin{example}
A scientist may assert that DrugA does not guarantee shrinkage, based on a $.56$-testable constraint:
$g(\vec{x}) := 1 - \texttt{treated}(mouse_1, \mathrm{DrugA}) - \texttt{shrink}(mouse_1) \ge 0.
$
This is witnessed  with probability $0.7 \cdot 0.8 = .56$ under $\Theta(D)$.
\end{example}

For moment constraints, we use confidence intervals derived from empirical averages. Because our system is explicitly compact, we can obtain these confidence intervals using Hoeffding's inequality:

\begin{theorem}[Hoeffding's Inequality]
    Let $X_1, \ldots, X_m$ be independent random variables such that $a \leq X_i \leq b$ for all $i$. Then for the empirical mean $\bar{X} = \frac{1}{m}\sum_{i=1}^m X_i$ and any $\varepsilon > 0$,
   $ \Pr[\bar{X} - \mathbb{E}[X] \geq \varepsilon] \leq \exp\left(-\frac{2m\varepsilon^2}{(b-a)^2}\right).$
\end{theorem}

\begin{definition}[Naive Norm]
Let $U_{\mu}$ and $L_{\mu}$ be the upper and lower bounds on the monomial $\mu$ as in an explicitly compact system. Given a polynomial expression $p(\vec{x})$, we define its \emph{naive upper bound} to be $U_p=\sum_{\mu:c_\mu \geq 0}c_\mu U_\mu+\sum_{\mu:c_\mu < 0}c_\mu L_\mu$ and its \emph{naive lower bound} to be $L_p=\sum_{\mu:c_\mu \geq 0}c_\mu L_\mu+\sum_{\mu:c_\mu < 0}c_\mu U_\mu$. Its naive norm is then $|p|:=U_p-L_p$.
\end{definition}

\begin{definition}[Testability of Expectation Constraints]
An expectation constraint $e(p)=\sum c_{e(\mu)}\,e(\mu) \ge \ell$ or $e(p)=\sum c_{e(\mu)}\,e(\mu) \le u$  is \emph{$(1-\gamma)$-testable} from $m$ partial examples if
it is statistically close to the empirical bound. That is, if it is an upper bound $e(p)\le u$, then $u \ge \bar{U}(p) - \sqrt{\frac{|p|^2\ln(\gamma/2)}{2m}}$ 
; if it is a lower bound $e(p)\ge l$, then $l\ge \bar{L}(p) + \sqrt{\frac{|p|^2\ln(\gamma/2)}{2m}}.$
The empirical bounds of a polynomial $p$ w.r.t.\ our monomial bounds is defined as $\bar{U}(p)=\frac{1}{m}\sum_t U_{p}(\vec\rho^{(t)})$ and $\bar{L}(p)=\frac{1}{m}\sum_t L_{p}(\vec\rho^{(t)})$. 
\end{definition}

\begin{example}~\label{ex:moment_testability}
A lab assistant gave a rough estimation instead:
$
e(\texttt{t}(mouse_1,\mathrm{DrugA}) \cdot \texttt{s}(mouse_1)) \le 0.5, \quad e(\texttt{t}(mouse_1,\mathrm{DrugA})) \ge 0.9.
$
Let us denote $t=\texttt{t}(mouse_1,\mathrm{DrugA}); s= \texttt{s}(mouse_1)$. If the empirical upper/lower bounds from $1000$ masked samples yield
$\bar{U}(ts)= 0.45$,  $\bar{L}(ts)= 0.25$  and $\bar{U}(t) = \bar{L}(t) = 1$,
then, by Hoeffding’s inequality, these estimates are $0.99$-testable.
\end{example}

\section{Algorithm for Inference with Implicit Learning and its Analysis}
In the usual way, we prove a query by introducing its negation to our knowledge base $\Delta$, and refuting the resulting, extended knowledge base. Thus, henceforth, we simply assume that our task is to refute a given collection $\Delta$.

In the algorithm, we will compute the tightest lower bound for each variable $v$ that is consistent with a given partial model $\vec\rho^{(i)}$ by determining the minimum value $v$ can take while satisfying all witnessed constraints in $\Delta$ restricted to $\vec\rho^{(i)}$, which we denote by $\mathrm{LowerBound}(v, \vec\rho^{(i)})$. Similarly, $\mathrm{UpperBound}(v, \vec\rho^{(i)})$ computes the maximum value. For directly observed variables in $\vec\rho^{(i)}$, these functions simply return the observed value.

\begin{algorithm}[h!]
\caption{Implicit Learning to Reason}
\begin{algorithmic}[1]
  \Require degree bound $d$, naive norm bound $S$, confidence parameter $\delta$, partial models $\vec\rho^{(1)}, \dots, \vec\rho^{(m)}$, KB $\Delta$, the quantifier rank $k$ of $\Delta$
  
\State Let \( V \) be the set of monomials of degree at most \( d \) formed over terms in \( \mathrm{GND}(\Delta,k) \).

  \State Initialize arrays $L[v]\gets 0$ and $U[v]\gets 0$ for all $v\in V$
  
  \For{$i = 1,\dots,m$}
    \ForAll{$v\in V$}
      \State $L[v] \gets L[v] + \mathrm{LowerBound}(v, \vec\rho^{(i)})$ \Comment{via SOS}
      \State $U[v] \gets U[v] + \mathrm{UpperBound}(v, \vec\rho^{(i)})$ \Comment{via SOS}
    \EndFor
  \EndFor

  \ForAll{$v\in V$}
    \State $L[v] \gets L[v] / m$  \Comment{average bounds}
    \State $U[v] \gets U[v] / m$
  \EndFor

  \State \textbf{Run} the SOS solver on $\{\Delta|_{\vec{\rho}^{(t)}}\}_{t\in[m]}$ with moment bounds $B=\{e(v)\ge L[v]-\frac{ S\sqrt{\ln({2n \choose \le d}/ \delta)}}{2m}\}_{v\in V}\cup \{e(v)\le U[v]+\frac{ S\sqrt{\ln({2n \choose \le d}/ \delta)}}{2m}\}_{v\in V}$
  \State \Return \textbf{True} if the SOS program is feasible, else \textbf{False}  
\end{algorithmic}
\end{algorithm}


\begin{theorem}[Soundness]
    Let $\delta \in (0,1)$ and $d,k \in \mathbb{N}$ be given and let $S$ be the input naive norm parameter. Let $n$ be the number of ground atoms in $GND(\Delta,k)$.  Suppose we have $m = \Omega(S^2((n+d)^{n-1} \log \frac{1}{\delta}))$ partial models $\vec\rho^{(1)}, \vec\rho^{(2)},...,\vec\rho^{(m)}$ drawn i.i.d. from a $\Theta(D)$ and all monomials have its norm $U_v-L_v\le S.$
    If $D$ satisfies the input system $\Delta$ then Algorithm 1 will return True with probability $1-\delta$.
\end{theorem}

\begin{proof}
Starting with the logical constraints of $\Delta$, we see that every $\vec\rho^{(t)}$ is consistent with $\Delta$. 
For the expectation bounds, we observe that every moment $e(v)$ is also a bounded random variable in [$L_{v},U_{v}$]. Then using Hoeffding's inequality, we can guarantee that $e(v)$ is within the estimated range of our empirical mean with probability $1-\delta/ {2n \choose \leq d}$. Since $e(v)$ is assumed to satisfy the constraints, its existence in our estimated range guarantees a feasible SOS program---meaning that we can take a union bound over all moments and produce satisfying ranges for all expectation constraints simultaneously with probability $1-\delta$. This implies that our SOS program is feasible, and the algorithm will return True.\end{proof}

\begin{theorem}[Completeness]
Let $\delta \in (0,1)$, $d,k \in \mathbb{N}$ be given, and  $S$ be the naive norm bound. Let $n$ be the number of ground atoms in $GND(\Delta,k)$. 
    Suppose  $m \in \Omega(S^2((n+d)^{n-1} \log \frac{1}{\delta}))$,  there is a set of logical constraints $I_s$ that is $(1-\frac{1}{2S})$-testable under $\Theta(D)$ and a set of expectation constraints $I_m$ that is $1-\delta/2$-testable under $\Theta(D)$, that, together with the input system $\Delta$ have a degree-$d$ sos-refutation with naive norm bounded by $S$.

    Then, with probability $1-\delta$, Algorithm 1 also returns a refutation.
\end{theorem}

\begin{proof}
By assumption, there is an SOS refutation over the union of the implicit knowledge base and explicit knowledge base $\Delta \cup I_s \cup I_m$, where $I_s$ is the set of logical constraints and $I_m$ is the set of expectation constraints:
\begin{equation}
\small
    \sigma_0 +  \sum_{k } \sigma_k g_k + \sum_{j }  c^+_j q_j = -1, \label{e:sos_hypothetical_refutation}
\end{equation}
where each $\sigma$ is a sum of squares, $\{g_k\}$ are all the logical constraints from the combined knowledge base and $\{q_j\}$ are all the expectation constraints, and the naive norm of the overall expression is $\le S$.

Note this is a formal polynomial identity, hence under any evaluations of the variables, the identity will hold.
Thus, if we take the average over the expression evaluated over $m$ examples, formally we will still have an identity with right hand side equal to $-1:$
\begin{multline}
\small
  \frac{1}{m}\sum_{t=1}^m
  \Bigl[
    \sigma_0|_{\vec\rho^{(t)}}
    +\sum_k \sigma_k|_{\vec\rho^{(t)}}\,g_k|_{\vec\rho^{(t)}}
    +\sum_j c_j^+\,q_j|_{\vec\rho^{(t)}}
  \Bigr]
 =-1.\nonumber
\end{multline}



Because $I_m$ is $1-\delta/2$ testable and each $q_j$ is linear in the expectations $e(v)$, and the naive norm of each $p_j$ is bounded by $S$, we can upper- or lower-bound each $q_j$ using the empirical bounds in $B$ with at most $\tfrac{1}{2}$ total slack across all constraints.


Therefore, we can replace the moment bound with the empirical bound up to the slack constant. Hence we still have the expression:
\begin{multline}
\small
    \frac{1}{m}\sum_{t=1}^m
  \Bigl[
    \sigma_0|_{\vec\rho^{(t)}}
    +\sum_k \sigma_k|_{\vec\rho^{(t)}}\,g_k|_{\vec\rho^{(t)}}
    + \\ \sum_{j'} c_{j'}^+\,B_{j'}\Bigr]
 =-1+\frac{1}{2}<0.
\end{multline}
By rescaling, we still have a refutation.

Next, by the testability of $I_s$, with probability $1-\delta/2$, for each $g_k,h_i\in I_s$, it is witnessed true for at least $\frac{2m}{S}$ examples, i.e.
$g_k|_{\vec\rho^{(t)}}\ge 0$. And for the rest of the examples, the naive norm 
of the entire proof is bounded by $-S$. Due to the testability of $I_s$, we know there are at most $\frac{m}{2S}$ examples unwitnessed. Using the explicit compactness constraints, the remaining terms from these unwitnessed expression can be bounded by their naive norms, which again sum to at most $S$ in total.

Hence we can subtract all the unwitnessed $g_k$ and $h_i$ on both sides, and get
\begin{multline}
\small
    \frac{1}{m}\sum_{t=1}^m
  \Bigl[
    \sigma_0|_{\vec\rho^{(t)}}
    +\sum_k \sigma_k|_{\vec\rho^{(t)}}\,g_k|_{\vec\rho^{(t)}}
    +\sum_{j'} c_{j'}^+\,B_{j'}\Bigr]
 =\\-1+\frac{m\cdot S}{m\cdot 2S}=-1+\frac{1}{2}<0.
\end{multline}
Rescaling again, we get another refutation.
Hence with probability $1-\delta$ overall, there also exists a SOS refutation over just
$\Delta_{\vec\rho}$ and the upper lower bounds for the moments, which is the system our algorithm is checking.

Our guarantee for this lifted knowledge base now follows from Theorem~\ref{thm:lifted-guarantee}.
\end{proof}
\begin{example}
    Continue with our example, the implicit, testable knowledge base from example~\ref{ex:moment_testability} first proves the premise of DrugA to end before next trial, i.e., there is a sum-of-square proof $[0.8e(t)-e(ts)]=0.8[e(t)-0.9]+[0.5-e(ts)]+0.22$. Then using the original knowledge base, we get $\texttt{second\_stage}(\mathrm{DrugA})=0$, hence we will refute the hypothesis that every drug will enter Phase II. 
    Moreover, using the more accurate bounds calculated by our algorithm will also lead to the same conclusion as
     $[0.8e(t)-e(ts)]=0.8[e(t)-1]+[0.44-e(ts)]+0.36.$
\end{example}


We now show that, for any fixed SOS degree $d$ and quantifier‐rank $k$, our lifted semidefinite‐programming (SDP) based inference algorithm runs in time polynomial in the \emph{bit‐complexity} of the input and of the SOS proof.  In particular, the potentially infinite domain and the number of partial examples only enter linearly (or not at all) in the exponent.


\begin{theorem}[Polynomial‐Time in Bit‐Complexity]
\label{thm:runtime}
Let $\Delta$ be an explicit FOPRL knowledge base whose encoding has bit‐length $B$, $\varphi$ be a (ground) query of bit‐length $b$, $d$ the chosen SOS degree bound, and $k$ the quantifier‐rank bound.
Then Algorithm 1 can be compiled into a single SDP whose matrix size and number of (PSD) constraints is at most
  $(B + b)^{O(d + k)}$
and whose bit‐precision requirement is likewise $(B + b)^{O(d + k)}$.  Hence the overall solver runs in
$O\bigl((B + b)^{c\,(d + k)}\bigr)$
bit‐operations for some absolute constant $c$.  In particular, with $d,k$ fixed, this is polynomial in $B + b$.
\end{theorem}

\begin{proof}[Proof Sketch]
\textbf{Grounding‐Lift.}  We identify all renaming‐equivalent degree$\le d$ monomials across the (possibly infinite) domain and collapse them into a single lifted moment variable.  This reduces the raw propositional SDP of size exponential in the domain to one of size only polynomial in $(B+b)^{d}$. \\
\textbf{World‐Lift.}  Instead of writing separate PSD constraints for each of the $m$ partial examples, we enforce a single lifted PSD matrix whose rows and columns correspond to the lifted monomials.  Each example contributes only linear‐size linear constraints (one per moment), so the total bit‐size remains polynomial in $B+b$ and $m$. \\
\textbf{Bit‐Complexity Control.}  All numerical data (coefficients in $\Delta$, sampling errors, etc.) have bit‐length at most $O(B+b)$ and appear in the SDP entries only polynomially many times.  Existing results on the bit‐complexity of convex optimization (e.g.\ \cite{grotschel2012geometric}) then imply an overall runtime bounded by $(B+b)^{O(d+k)}$.
\end{proof}

A direct propositional SOS encoding would require distinct moment variables for each ground monomial, yielding an SDP of size \(\Omega(|\mathcal{U}|^d)\), where \(\mathcal{U}\) is the (possibly infinite) universe.  Our double‐lifting avoids this blow‐up entirely.

\section{Conclusion and Limitations}

In this work, we presented a polynomial-time framework for implicitly learning to reason in first-order relational probabilistic logic, advancing beyond previous work limited in pure inference. Our approach unifies incomplete first-order axioms with partial observations through a bounded-degree SOS hierarchy, enabling tractable inference over infinite domains with quantified variables. 

Our approach simultaneously handles first-order symmetries and probabilistic semantics, without imposing restrictions on predicate arity or clause length. Because the method avoids reliance on explicit models or independence assumptions, it naturally accommodates complex noise and interdependent variables, with flexibility as a core strength.

While our assumption of bounded variables excludes certain distributions, we note that all real-world data is finitely represented. Even with very large bounds, our method remains applicable.

Additionally, we recognize the SOS solvers remain an area of active research \citep{huang2022solving,jiang2023faster}. The presence of constants and reliance on moderately low-degree SOS proofs may pose challenges for scaling to large, richly relational domains, requiring further algorithmic and engineering advances.

We do not claim empirical scalability in this paper. However, the practical viability of our approach is relevant for future work, and we emphasize that existing tools—such as SparsePOP \citep{waki2006sums}—are capable of solving the kinds of instances required by our framework.

Our polynomial‐time framework for combining partial observations with incomplete first-order background knowledge offers reliable decision support in data-scarce domains—from early-stage biomedical trials to environmental monitoring. By making uncertainty explicit via witnessed constraints and confidence bounds, it also lays the groundwork for risk-aware systems. Although still theoretical, owe plan to pen-source our SOS proof certificates to support transparency, reproducibility, and community-driven validation on real-world challenges.

\bibliography{references}

@inproceedings{belle2019,
 author = {Belle, Vaishak and Juba, Brendan},
 booktitle = {Advances in Neural Information Processing Systems},
 editor = {H. Wallach and H. Larochelle and A. Beygelzimer and F. d\textquotesingle Alch\'{e}-Buc and E. Fox and R. Garnett},
 pages = {},
 publisher = {Curran Associates, Inc.},
 title = {Implicitly learning to reason in first-order logic},
 
 volume = {32},
 year = {2019}
}

@phdthesis{parrilo00,
 author = {Pablo A. Parrilo},
 title = {Structured semidefinite programs and semialgebraic geometry methods in robustness and optimization},
 school = {California Institute of Technology},
 year = {2000}
}

@article{lasserre01,
 author = {Jean Bernard Lasserre},
 title = {Global Optimzation with Polynomials and the Problem of Moments},
 journal = {SIAM J. Optimization},
 volume = {11},
 number = {3},
 pages = {796--817},
 year = {2001}
}

@article{nesterov00,
 author = {Y. Nesterov},
 title = {Squared functional systems and optimization problems},
 journal = {High performance optimization},
 volume = {13},
 pages = {405--440},
 year = {2000}
}

@article{shor87,
 author = {N. Shor},
 title = {An approach to obtaining global extremums in polynomial mathematical programming problems},
 journal = {Cybernetics and Systems Analysis},
 volume = {23},
 number = {5},
 pages = {695--700},
 year = {1987}
}

@article{putinar,
 author = {M. Putinar},
 title = {Positive polynomials on compact semi-algebraic sets},
 journal = {Indiana U. Math. J.},
 volume = {42},
 pages = {969--984},
 year = {1993}
}

@article{ghasemi2016moment,
  title={Moment problem in infinitely many variables},
  author={Ghasemi, Mehdi and Kuhlmann, Salma and Marshall, Murray},
  journal={Israel Journal of Mathematics},
  volume={212},
  pages={1012--1012},
  year={2016},
  publisher={Springer}
}

@article{cropper2022inductive,
  title={Inductive logic programming at 30: a new introduction},
  author={Cropper, Andrew and Duman{\v{c}}i{\'c}, Sebastijan},
  journal={Journal of Artificial Intelligence Research},
  volume={74},
  pages={765--850},
  year={2022}
}

@inproceedings{juba2013implicit,
  title={Implicit Learning of Common Sense for Reasoning.},
  author={Juba, Brendan},
  booktitle={IJCAI},
  pages={939--946},
  year={2013}
}

@inproceedings{marra2021neural,
  title={Neural markov logic networks},
  author={Marra, Giuseppe and Ku{\v{z}}elka, Ond{\v{r}}ej},
  booktitle={Uncertainty in Artificial Intelligence},
  pages={908--917},
  year={2021},
  organization={PMLR}
}

@book{koller2009probabilistic,
  title={Probabilistic graphical models: principles and techniques},
  author={Koller, Daphne and Friedman, Nir},
  year={2009},
  publisher={MIT press}
}

@incollection{bruynooghe2010problog,
  title={Problog technology for inference in a probabilistic first order logic},
  author={Bruynooghe, Maurice and Mantadelis, Theofrastos and Kimmig, Angelika and Gutmann, Bernd and Vennekens, Joost and Janssens, Gerda and De Raedt, Luc},
  booktitle={ECAI 2010},
  pages={719--724},
  year={2010},
  publisher={IOS Press}
}

@article{richardson2006markov,
  title={Markov logic networks},
  author={Richardson, Matthew and Domingos, Pedro},
  journal={Machine learning},
  volume={62},
  pages={107--136},
  year={2006},
  publisher={Springer}
}

@inproceedings{de2007problog,
  title={ProbLog: A probabilistic Prolog and its application in link discovery},
  author={De Raedt, Luc and Kimmig, Angelika and Toivonen, Hannu},
  booktitle={IJCAI 2007, Proceedings of the 20th international joint conference on artificial intelligence},
  pages={2462--2467},
  year={2007}
}

@inproceedings{rader2021learning,
  title={Learning Implicitly with Noisy Data in Linear Arithmetic},
  author={Rader, Alexander P and Mocanu, Ionela G and Belle, Vaishak and Juba, Brendan},
  booktitle={30th International Joint Conference on Artificial Intelligence},
  pages={1410--1417},
  year={2021}
}

@incollection{mocanu2020polynomial,
  title={Polynomial-time implicit learnability in SMT},
  author={Mocanu, Ionela G and Belle, Vaishak and Juba, Brendan},
  booktitle={ECAI 2020},
  pages={1152--1158},
  year={2020},
  publisher={IOS Press}
}

@inproceedings{juba2019polynomial,
  title={Polynomial-time probabilistic reasoning with partial observations via implicit learning in probability logics},
  author={Juba, Brendan},
  booktitle={Proceedings of the AAAI Conference on Artificial Intelligence},
  volume={33},
  number={01},
  pages={7866--7875},
  year={2019}
}

@inproceedings{belle2017open,
  title={Open-universe weighted model counting},
  author={Belle, Vaishak},
  booktitle={Proceedings of the AAAI Conference on Artificial Intelligence},
  volume={31},
  number={1},
  year={2017}
}

@inproceedings{lakemeyer2002evaluation,
  title={Evaluation-based reasoning with disjunctive information in first-order knowledge bases},
  author={Lakemeyer, Gerhard and Levesque, Hector J},
  booktitle={KR},
  pages={73--81},
  year={2002}
}

@inproceedings{van2011lifted,
  title={Lifted probabilistic inference by first-order knowledge compilation},
  author={Van den Broeck, Guy and Taghipour, Nima and Meert, Wannes and Davis, Jesse and De Raedt, Luc},
  booktitle={IJCAI},
  pages={2178--2185},
  year={2011}
}

@article{halpern2007characterizing,
  title={Characterizing and reasoning about probabilistic and non-probabilistic expectation},
  author={Halpern, Joseph Y and Pucella, Riccardo},
  journal={Journal of the ACM (JACM)},
  volume={54},
  number={3},
  pages={15--es},
  year={2007},
  publisher={ACM New York, NY, USA}
}

@inproceedings{ge2025polynomial,
  title={Polynomial-Time Relational Probabilistic Inference in Open Universes},
  author={Ge, Luise and Juba, Brendan and Nilsson, Kris},
  booktitle={Proceedings of the Thirty-Fourth International Joint Conference on Artificial Intelligence (IJCAI-25)},
  year={2025},
  
}

@article{manhaeve2018deepproblog,
  title={Deepproblog: Neural probabilistic logic programming},
  author={Manhaeve, Robin and Dumancic, Sebastijan and Kimmig, Angelika and Demeester, Thomas and De Raedt, Luc},
  journal={Advances in neural information processing systems},
  volume={31},
  year={2018}
}

@incollection{getoor2001learning,
  title={Learning probabilistic relational models},
  author={Getoor, Lise and Friedman, Nir and Koller, Daphne and Pfeffer, Avi},
  booktitle={Relational data mining},
  pages={307--335},
  year={2001},
  publisher={Springer}
}

@misc{bubeck2023sparks,
  title={Sparks of Artificial General Intelligence: Early experiments with GPT-4},
  author={Bubeck, S{\'e}bastien and Chandrasekaran, Varun and Eldan, Ronen and Gehrke, Johannes and Horvitz, Eric and Kamar, Ece and Lee, Peter and Lee, Yin Tat and Li, Yuanzhi and Lundberg, Scott and others},
  howpublished={arXiv:2303.12712},
  year={2023}
}

@inproceedings{dziri2023faith,
title={Faith and Fate: Limits of Transformers on Compositionality},
author={Nouha Dziri and Ximing Lu and Melanie Sclar and Xiang Lorraine Li and Liwei Jiang and Bill Yuchen Lin and Sean Welleck and Peter West and Chandra Bhagavatula and Ronan Le Bras and Jena D. Hwang and Soumya Sanyal and Xiang Ren and Allyson Ettinger and Zaid Harchaoui and Yejin Choi},
booktitle={Thirty-seventh Conference on Neural Information Processing Systems},
year={2023}
}

@misc{mccoy2023embers,
  title={Embers of autoregression: Understanding large language models through the problem they are trained to solve},
  author={McCoy, R Thomas and Yao, Shunyu and Friedman, Dan and Hardy, Matthew and Griffiths, Thomas L},
  howpublished={arXiv:2309.13638},
  year={2023}
}

@inproceedings{gendronlarge,
  title={Large Language Models Are Not Strong Abstract Reasoners},
  author={Gendron, Ga{\"e}l and Bao, Qiming and Witbrock, Michael and Dobbie, Gillian},
  booktitle = {Proc.\ 33rd IJCAI},
  pages = {6270-6278},
  year = {2024}
}

@misc{guo2025deepseek,
  title={Deepseek-r1: Incentivizing reasoning capability in llms via reinforcement learning},
  author={Guo, Daya and Yang, Dejian and Zhang, Haowei and Song, Junxiao and Zhang, Ruoyu and Xu, Runxin and Zhu, Qihao and Ma, Shirong and Wang, Peiyi and Bi, Xiao and others},
  howpublished={arXiv:2501.12948},
  year={2025}
}

@article{kambhampati2025reasoning,
  title={(How) Do reasoning models reason?},
  author={Kambhampati, Subbarao and Stechly, Kaya and Valmeekam, Karthik},
  journal={Annals of the New York Academy of Sciences},
  year={2025},
  publisher={Wiley Online Library}
}

@misc{mccoy2024language,
  title={When a language model is optimized for reasoning, does it still show embers of autoregression? {A}n analysis of {OpenAI} o1},
  author={McCoy, R Thomas and Yao, Shunyu and Friedman, Dan and Hardy, Mathew D and Griffiths, Thomas L},
  howpublished={arXiv:2410.01792},
  year={2024}
}

@inproceedings{rocktaschel2017end,
  title={End-to-end Differentiable Proving},
  author={Rockt{\"a}schel, Tim and Riedel, Sebastian},
  booktitle={Advances in Neural Information Processing Systems},
  pages={3788--3800},
  year={2017}
}

@article{donadello2017logic,
  title={Logic Tensor Networks for Semantic Image Interpretation},
  author={Donadello, Ivan and Serafini, Luciano and d'Avila Garcez, Artur},
  journal={Proceedings of the 26th International Joint Conference on Artificial Intelligence},
  pages={1596--1602},
  year={2017}
}

@article{jiang2023faster,
  title={A faster interior-point method for sum-of-squares optimization},
  author={Jiang, Shunhua and Natura, Bento and Weinstein, Omri},
  journal={Algorithmica},
  volume={85},
  number={9},
  pages={2843--2884},
  year={2023},
  publisher={Springer}
}

@inproceedings{huang2022solving,
  title={Solving sdp faster: A robust ipm framework and efficient implementation},
  author={Huang, Baihe and Jiang, Shunhua and Song, Zhao and Tao, Runzhou and Zhang, Ruizhe},
  booktitle={2022 IEEE 63rd Annual Symposium on Foundations of Computer Science (FOCS)},
  pages={233--244},
  year={2022},
  organization={IEEE}
}

@book{grotschel2012geometric,
  title={Geometric algorithms and combinatorial optimization},
  author={Gr{\"o}tschel, Martin and Lov{\'a}sz, L{\'a}szl{\'o} and Schrijver, Alexander},
  volume={2},
  year={2012},
  publisher={Springer Science \& Business Media}
}

@article{waki2006sums,
  title={Sums of squares and semidefinite program relaxations for polynomial optimization problems with structured sparsity},
  author={Waki, Hayato and Kim, Sunyoung and Kojima, Masakazu and Muramatsu, Masakazu},
  journal={SIAM Journal on Optimization},
  volume={17},
  number={1},
  pages={218--242},
  year={2006},
  publisher={SIAM}
}

\end{document}